\documentclass[letterpaper, 10 pt, conference]{ieeeconf}
\IEEEoverridecommandlockouts                              

\usepackage{cite}

\usepackage{amsmath,amssymb,amsfonts, mathtools, amsthm}
\usepackage{array}

\usepackage[caption=false]{subfig}

\usepackage{textcomp}
\usepackage{stfloats}
\usepackage{url}
\usepackage{verbatim}
\usepackage{graphicx}
\usepackage[noend]{algpseudocode}
\usepackage{algorithm}
\def\BibTeX{{\rm B\kern-.05em{\sc i\kern-.025em b}\kern-.08em
    T\kern-.1667em\lower.7ex\hbox{E}\kern-.125emX}}
\usepackage{balance}
\usepackage{bm}

\theoremstyle{plain}
   
\theoremstyle{remark}
\newtheorem{remark}{Remark}
\theoremstyle{definition}
\newtheorem{definition}{Definition}
\theoremstyle{plain}
\newtheorem{theorem}{Theorem}
\newtheorem{observation}{Observation}
\newtheorem{corollary}{Corollary}

\newtheorem{problem}{Problem}
\newtheorem{lemma}{Lemma}
 
\newcommand{\Mod}[1]{\ (\mathrm{mod}\ #1)}
\usepackage{scalerel}
\usepackage{xcolor}

\usepackage[para, online, flushleft]{threeparttable}
\usepackage{array,booktabs,makecell}
\usepackage{bm}
\let\labelindent\relax
\usepackage[inline]{enumitem}
\usepackage{booktabs}
\usepackage{relsize}

\title{\bf Risk-aware Resource Allocation for \\Multiple UAVs-UGVs Recharging Rendezvous
\thanks{This work is supported in part by National Science Foundation Grant No. 1943368 and Army Grant No. W911NF2120076.}
\thanks{$^{*}$ indicates equal contribution and authors are listed alphabetically}}

\author{Ahmad Bilal Asghar,$^{1*}$ Guangyao Shi,$^{1*}$ Nare Karapetyan,$^1$ James Humann,$^2$ \\ Jean-Paul Reddinger,$^2$ James Dotterweich,$^2$ Pratap Tokekar$^1$ 
\thanks{\textsuperscript{1}University of Maryland, College Park, MD 20742 USA {\tt\small gyshi}, {\tt\small knare}, {\tt\small abasghar}, {\tt\small tokekar@umd.edu}}
\thanks{\textsuperscript{2}DEVCOM Army Research Laboratory, MD USA {\tt\small jean-paul.f.reddinger.civ}, {\tt\small james.d.humann.civ}, {\tt\small james.m.dotterweich.civ@army.mil}}
}

\begin{document}
\maketitle
\thispagestyle{empty}
\pagestyle{empty}
\begin{abstract}
We study a resource allocation problem for the cooperative aerial-ground vehicle routing application, in which multiple Unmanned Aerial Vehicles (UAVs) with limited battery capacity and multiple Unmanned Ground Vehicles (UGVs) that can also act as a mobile recharging stations need to jointly accomplish a
mission such as persistently monitoring a set of points. Due to the limited
battery capacity of the UAVs, they sometimes have to
deviate from their task to rendezvous with the UGVs and get recharged. Each UGV can serve a limited number of UAVs at a time. In contrast to prior work on deterministic multi-robot scheduling, we consider the challenge imposed by the stochasticity of the energy consumption of the UAV. We are interested in finding the optimal recharging schedule of the UAVs such that the travel cost is minimized and the probability that no UAV runs out of charge within the planning horizon is greater than a  user-defined tolerance.  We formulate this problem ({Risk-aware Recharging Rendezvous Problem (RRRP))} as an Integer Linear Program (ILP), in which the matching constraint captures the resource availability constraints and the knapsack constraint captures the success probability constraints. We propose a bicriteria approximation algorithm  to solve RRRP.  We demonstrate the effectiveness of our
formulation and algorithm in the context of one persistent monitoring mission.   
\end{abstract}

\section{Introduction}
Unmanned Aerial Vehicles (UAVs) are increasingly being used in applications such as surveillance \cite{kingston2008decentralized, grocholsky2006cooperative}, package delivery \cite{lin2011vehicle, murray2015flying},  environmental
monitoring \cite{diaz2012volcano, sung2021environmental}, and precision agriculture \cite{dhami2020crop} due to their ability to monitor large areas in a short period of time. One bottleneck that hinders  long-term deployments of UAVs is their limited battery capacity. Recently, there have been efforts in overcoming this bottleneck by using Unmanned Ground Vehicles (UGVs) as mobile recharging stations~\cite{tokekar2016sensor, shi2022risk, yu2018algorithms, mathew2015multirobot, maini2019cooperative, maini2015cooperation, sundar2013algorithms, liu2014energy, jeong2019truck, karak2019hybrid, li2021ground}. However, these works make the simplifying but restrictive assumption that the energy consumption of the UAV is deterministic. In practice, the energy consumption is stochastic and the planning algorithms must be able to deal with the risk of running out of charge. In our recent work~\cite{shi2022risk}, we presented a risk-aware planning algorithm for planning for a single UAV and UGV. However, that algorithm scales exponentially with the planning horizon and the number of robots. In this paper, we present an efficient risk-aware coordination algorithm for scalable, long-term UAV-UGV missions. 

We consider a scenario where the UAVs and the UGVs are executing a persistent monitoring mission by visiting a sequence of \emph{task} nodes in the order given (Figure~\ref{fig:illustrative_example}). The UAVs can take a detour from the planned mission to rendezvous with some UGV and recharge. The UGV can recharge the UAV while continuing towards its own next task node. We present several algorithms that decide \emph{when} and \emph{where} the UAVs should  recharge and \emph{which} UGV they should recharge on. 

The rate of battery discharge of a UAV is stochastic in the real world. The planning algorithms should be able to deal with such stochasticity. One approach would be to take recharging detours more frequently, thereby reducing the risk of running out of charge\footnote{In practice, when a UAV is about to run out of charge, it can land and then be retrieved by a human. By reducing the risk of running out of charge, we aim to reduce the overhead of such costly human interventions.}. However, if a UAV takes a detour, then the task nodes visited after the detour will experience a delay, thereby worsening the task performance. Since we are planning over a longer horizon with multiple UAVs and UGVs, there is a complex trade-off
between task performance and risk tolerance.  

\begin{figure}
    \centering
    \includegraphics[width=1\columnwidth]{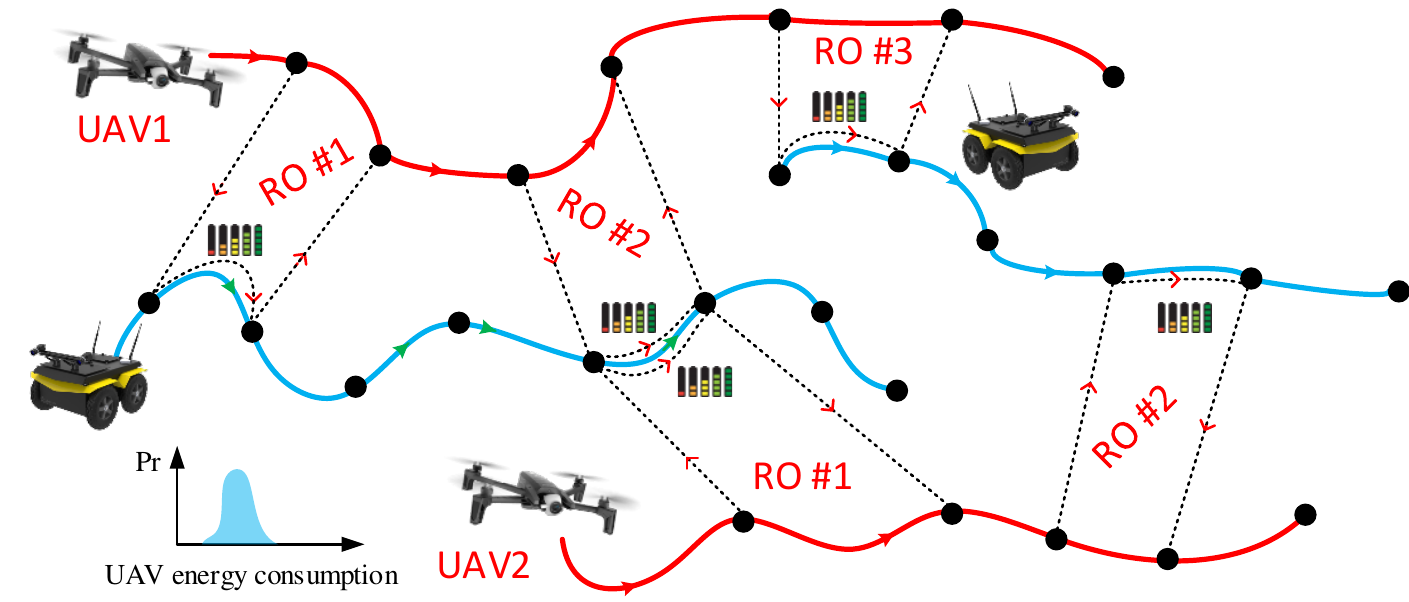}
    \caption{An illustrative example of the recharging rendezvous problem considered in this paper. A team of two UAVs and two UGVs is executing their tasks in the given order. The UAVs need to decide when and where to land on the UGVs in order to recharge minimizing the total detour time.  The energy consumption of the UAVs is stochastic. RO stands for rendezvous option.}
    \label{fig:illustrative_example}
\end{figure}

Given a team of UGVs that move
on the road network, and UAVs that can directly fly between
any pair of nodes, we are interested in finding a scheduling strategy for the UAVs by treating UGVs as resources such that the sum of detour cost of the UAVs is minimized and the probability that no UAV runs out of charge within the planning horizon is greater than a user-defined tolerance. We formulate this problem as a resource allocation problem with matching and knapsack constraints. The matching constraint captures the resource
availability constraints (each UGV can serve a limited number of
UAVs at a time). The knapsack constraint captures the risk-tolerance requirement (probability that no UAV runs out of charge).

The Risk-aware Recharging Rendezvous
Problem (RRRP) can be formulated as an integer linear program (ILP). While there are sophisticated ILP solvers, we show how to exploit the structural properties of the problem to yield more efficient algorithms with theoretical performance guarantees. Other combinatorial problems with knapsack or budget constraints have been studied in literature. Approximation algorithms for the budgeted version of the minimum spanning tree problem and maximum weight matching problem are presented in~\cite{ravi1996constrained} and~\cite{berger2011budgeted} respectively. They both use Lagrangian relaxation which resembles our approach, however, due to different problem structures, we propose a bicriteria approximation algorithm to solve RRRP.
We demonstrate the effectiveness of our
formulation and algorithm in the context of a persistent monitoring mission.  

Our work is built on results from the previous work on the cooperative routing problems, which are usually formulated as variants of the Vehicle Routing Problem (VRP) or the Traveling Salesman Problem (TSP) \cite{murray2015flying, tokekar2016sensor, lin2011vehicle, maini2019cooperative, yu2018algorithms, liu2019cooperative, manyam2016path, manyam2017cooperative}. In these works, authors usually assume a deterministic model for the energy consumption of the UAV or use the expected value for simplification. However, when the UAV executes the route, the actual energy consumption may be quite different from that used in route planning. Moreover, in the long-duration task, the estimation module may periodically update the energy consumption model. Given the routes of UAVs and UGVs generated by some existing algorithms, we consider the recharging scheduling problem within a planning horizon. Our formulation can be used in a receding horizon fashion to deal with the uncertainties in the long-duration tasks. More details will be discussed in the Sec \ref{sec:exp}.

The main contributions of this paper are:
\begin{itemize}
    \item We introduce the Risk-aware Recharging
    Rendezvous Problem (RRRP) as the first risk-aware version of the multi-robot recharging problem. We formulate the RRRP in the cooperative aerial-ground system as a matching problem on a bipartite graph with an additional knapsack constraint.
    \item We present several algorithms to solve the problem, that build on each other. Our main theoretical result is a $(1+\epsilon, 2)$ bicriteria approximation algorithm that yields a solution that violates one of the constraints by a factor of $2$, and has a detour cost within $1+\epsilon$ of the optimal cost.
    \item We demonstrate the effectiveness of our formulation and the proposed algorithm in a persistent monitoring application.
\end{itemize}



\section{Finite Horizon Risk-Aware Recharging Rendezvous Problem}

In this section we first define the Risk-aware Recharging
Rendezvous Problem on a finite horizon, and then formulate the RRRP as a matching problem on a bipartite graph with an additional knapsack constraint which enforces that the probability of no UAV running out of charge within the horizon is greater than a risk-tolerance parameter.

\subsection{Problem Statement}
Consider a team of $N_a$ UAVs and $N_g$ UGVs persistently monitoring a set of locations in an environment. The UAVs and UGVs move on the graphs $G_a = (U_a,E_a)$ and $G_g=(U_g,E_g)$ respectively. The vertex sets $U_a$ and $U_g$ represent the locations to be monitored by the aerial and ground vehicles respectively. The edge set $E_a$ may be complete since the UAVs can move between any of the tasks, whereas the edge set $E_g$ represents the road network on which the ground vehicles can move. We assume that the ordering of the task nodes for the UAVs and the UGVs is given, i.e., we have pre-defined persistent monitoring tours for the UAVs and UGVs. These tours can be generated by planners that either do not consider recharging~\cite{nigam2008persistent,asghar2019multi,hari2019efficient} or assume deterministic discharge~\cite{sundar2013algorithms,maini2019cooperative}. The tour for $i^{th}$ UAV is denoted by $\mathcal{T}^a_i$ and the tour for UGV $k$ is denoted by $\mathcal{T}^g_k$. Because of the persistent nature of the problem, the UAVs need to be recharged repeatedly in order for them to keep operating. 

A UAV $i$ can leave its monitoring tour $\mathcal{T}^a_i$ at any point along the tour to land on one of the UGVs, say UGV $k$, for recharging, while the UGV keeps moving along its tour $\mathcal{T}^g_k$. The UAV can also wait at a rendezvous location for the UGV if it reaches there before the UGV. The number of UAVs that can simultaneously charge on a given UGV is $d$. Once recharged, the UAV leaves the UGV and goes to the next task node along its monitoring tour. This procedure of a UAV leaving its monitoring tour to rendezvous with a UGV for recharging and then returning to continue its monitoring tour is referred to as a \emph{detour}. 

We assume that the UGVs do not run out of charge.\footnote{{This is a standard assumption in the literature~\cite{mathew2015multirobot, tokekar2016sensor, yu2018algorithms, maini2019cooperative, manyam2016path, manyam2019cooperative} since the runtime of the UGV can be an order of magnitude larger than the UAV. Furthermore, the UGV battery (or fuel, if its fuel powered) can be easily and quickly replenished unlike the UAV.}} Unlike much of the existing work~\cite{maini2019cooperative, mathew2015multirobot, yu2018algorithms}, we do not assume that the discharge rate of by the UAVs is deterministic. Instead, we consider a stochastic discharge model and assume that the probability of a UAV running out of charge within the next $t$ time units given the current charge level can be calculated as that in \cite{shi2022risk}. The stochastic battery discharge is monotonic in the time traveled by the UAV.

Since it is a persistent monitoring scenario where the UAVs need to be recharged indefinitely, we take the approach of solving the problem in a receding-horizon manner. The problem within a given time window $T$ is to find a recharging policy for the UAVs for up to a single recharge for each UAV. Moreover, as the battery discharge rate of UAVs is stochastic, there may be a non-zero probability of some UAVs not being able to reach any of the recharging UGVs. Hence, we also need to have a notion of risk-aversion for the UAVs. On the other hand, to avoid frequent recharging, we also consider the detour time spent by the UAVs for recharging. 

{At a high level, we consider the following problem.}

\vspace{2mm}
\noindent \textit{Given a time horizon $T$, a risk-tolerance probability ${\rho}\in(0,1)$, task routes for the $N_a$ UAVs and $N_g$ UGVs along with their current locations and state of charge, we seek to solve the problem of finding a recharging schedule for each UAV such that:
\begin{enumerate}
     \item each UAV recharges at most once during $T$,
     \item no UGV can charge more than $d$ UAVs at a time, 
    \item the probability that no UAV runs out of charge during $T$ is at least ${\rho}$, and
    \item the total detour time of the UAVs is minimized.
\end{enumerate}
}
\vspace{2mm}

In the following paragraphs, we provide more details regarding the Finite Horizon Recharging Rendezvous Problem and then propose an algorithm to solve the problem in the next section.

\subsection{Integer Linear Program Formulation}
We now formally define the setup of the problem. Since, a UAV can decide at any point along its tour to leave the path in order to rendezvous with a UGV at any point along the UGVs tour, we have a continuous constrained optimization problem. We discretize the tours of the UAVs and UGVs in order to get a combinatorial optimization problem. Given the tour $\mathcal{T}^g_k$ for UGV $k$, we discretize the tour by introducing vertices every $f$ distance starting from the UGV's current position where $f$ is the maximum distance the UGV needs to travel while a UAV completes its charging. These vertices representing possible rendezvous locations are denoted by $V(\mathcal{T}^g_k)$ and are shown as green discs in Figure~\ref{fig:rendezvous_recharging_process}.
Similarly the set $V(\mathcal{T}^a_i)$ represents the set of locations from which UAV $i$ can leave its monitoring tour $\mathcal{T}^a_i$ for a recharging rendezvous with a UGV. The set $V(\mathcal{T}^a_i)$ contains the current position of UAV $i$ and its task nodes that can be visited within next $T$ time by UAV $i$. We do not need to further discretize UAVs tours as it can be shown (Lemma~\ref{lem:UAV_vertices}) that there exists an optimal solution where the UAVs will leave their tours for recharging from either their current position or a task node. In Figure~\ref{fig:rendezvous_recharging_process}, the vertices in $V(\mathcal{T}^a_i)$ are shown as red discs. In the scenario shown in Figure~\ref{fig:rendezvous_recharging_process} the UAV 1 chooses to leave its task route at the node $a^1_1 \in V(\mathcal{T}^a_1)$ in order to rendezvous with UGV 1 at vertex $g_1^3\in V(\mathcal{T}^g_1)$. The recharging is complete when the UGV reaches node $g_4^1$ and the UAV moves to the next task node $a_1^2$ in its tour $\mathcal{T}^a_i$. Note that if $d=1$, no other UAV can recharge on this UGV between $g_1^3$ and $g_1^4$.

We now define the problem formally on a bipartite graph $G=(V_a\cup V_g, E)$, defined as follows. 

\begin{figure}
    \centering
    \includegraphics[width=0.35 \textwidth]{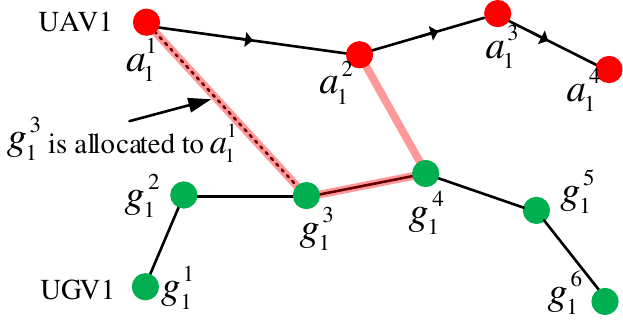}
    \caption{
     Example of a recharging detour. The UAV 1 leaves its task route at the node $a^1_1$ to rendezvous with UGV 1 at the node $g_1^3$. The recharging is done when they reach the node $g_1^4$ and  the UAV moves to the next task node $a_1^2$. }
    \label{fig:rendezvous_recharging_process}
\end{figure}

\begin{itemize}
    \item \emph{UAV vertices:} The vertex set $V_a$ consists of $N_a$ disjoint node sets,  i.e., $V_a = \cup_{i=1}^{N_a}\mathcal{V}_i$ where $\mathcal{V}_i=V(\mathcal{T}^a_i) \cup a^{\emptyset}_i$. The vertex $a^{\emptyset}_i$ represents the scenario where the UAV $i$ chooses not to rendezvous in the current horizon.
    \item \emph{UGV vertices:} The vertex set $V_g$ consists of $\{g^{\emptyset}_1, \ldots, g^{\emptyset}_i, \ldots, g^{\emptyset}_{N_a}\}$ and $d$ copies of the set $\cup_{k=1}^{N_g}  V(\mathcal{T}^g_k)$. The vertex $g^{\emptyset}_i$ for UAV $i$ represents the scenario where UAV $i$ chooses not to rendezvous for the current horizon. The $d$ copies of each vertex in $V(\mathcal{T}^g_k)$ are to represent up to $d$ different UAVs recharging at a time on a UGV.\footnote{We use $d$ copies of each vertex in $V(\mathcal{T}^g_k)$ to keep the analysis of the algorithm simple. Note that the problem can be defined without having $d$ copies for each vertex in $V(\mathcal{T}^g_k)$ by changing Constraint~\eqref{eq:c33} in Problem~\ref{pbm:original} to $\sum_{i}x_{ij}\leq d$.}  
    \item \emph{Edges:} The edge set $E$ denotes the set of all feasible recharging detour options. If the UAV $i$, starting from a node $j\in V(\mathcal{T}^a_i)$, is able to rendezvous with the UGV $k$ at $l\in V(\mathcal{T}^g_k)$, an edge exist between the corresponding two nodes in $V_a$ and $V_g$. The vertex $a^{\emptyset}_i$ is connected to $g^{\emptyset}_i$ for all $i\in[N_a]$.
    \item \emph{Edge cost:} For an edge $(i,j)\in E$ where $i\in V_a$ and $j\in V_g$, the edge cost $c_{ij}$ represents the total time needed to finish the task route given that the recharging detour along edge $(i,j)$ is taken. The time needed for the recharging detour consists of three parts: the time to reach the rendezvous node, the waiting time and the recharging time at the rendezvous node, and the time to go to the next node on the tour. The edge cost along edge $(a^{\emptyset}_i,g^{\emptyset}_i)$  is zero.
    \item \emph{Edge success probability:} For an edge $(i,j)\in E$ where $i\in V_a$ and $j\in V_g$, the edge success probability $p_{ij}$ is defined as the overall probability to finish the task route for the current horizon given that the recharging detour along edge $(i,j)$ is taken. It is the product of the two success probabilities: the probability of a successfully completing the recharging detour, and the probability of finishing the rest of the route after recharging.  The probability along edge $(a^{\emptyset}_i,g^{\emptyset}_i)$ is the probability of reaching the end of the current horizon without recharging.
\end{itemize}

\begin{figure}
    \centering
    \includegraphics[width=0.30\textwidth]{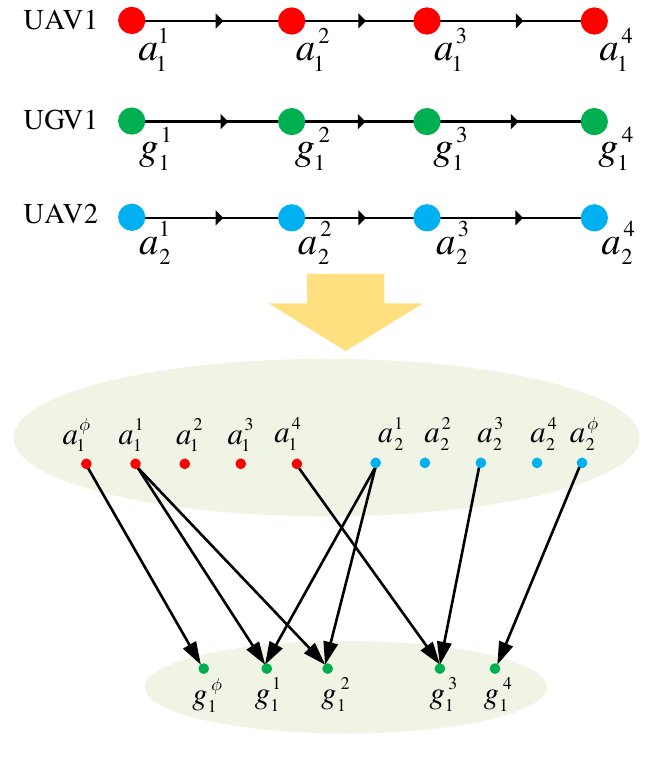}
    \caption{
     An example to show how to construct a bipartite graph for one planning horizon from the UAV and UGV route segments. }
    \label{fig:bipartite_graph}
\end{figure}

One example of the constructed bipartite graph with $d=1$ is shown in Figure \ref{fig:bipartite_graph}.

The finite horizon recharge rendezvous problem can now be defined as the following Integer Linear Program (ILP).

\begin{problem}[Risk-aware Recharging Rendezvous
Problem (RRRP)]
Given a bipartite graph $G=(V_a\cup V_g,E)$ where $V_a = \cup_{i=1}^{N_a}\mathcal{V}_i$, with edge costs $c_{ij}$ and probabilities $p_{ij}$ for edge $(i,j)\in E$, solve:
\label{pbm:original}
\begin{align}
     \min &~\sum_{i,j} c_{ij}x_{ij}& \\
    \text{s.t.}&~\sum_{i,j}  \log{\frac{1}{p_{ij}}} x_{ij} \leq \log{\frac{1}{\rho}}& \label{eq:knapsack1}\\
    &~\sum_{i\in \mathcal{V}_r}\sum_{j} x_{ij} = 1, \quad &\forall r\in [N_a]  \label{eq:c22}\\
    &~\sum_{i} x_{ij} \leq 1,\quad &\forall j\in V_g \label{eq:c33}\\ 
    &~x_{{ij}} \in \{0, 1\}&\label{eq:integral}
\end{align}
\end{problem}

The decision variable $x_{ij}$ indicates whether edge $(i,j)$ is in the solution or not. The objective is to minimize the travel time overhead incurred by the recharging detours. Inequality~\eqref{eq:knapsack1} enforces that the probability that no UAV runs out of charge during the current horizon is at least $\rho$. We can write this as a linear constraint since the stochastic discharging processes of the UAVs are independent. For ease of notation, we will use the following inequality instead of Constraint~\eqref{eq:knapsack1}.
\begin{equation}
\label{eq:knapsack}
    \sum_{i,j}a_{ij}x_{ij}\leq B 
\end{equation}
where $B=\log{1/\rho}$ and $a_{ij}=\log{1/p_{ij}}$.
Constraint~\eqref{eq:c22} enforces that each UAV should be recharged at most once. Constraint~\eqref{eq:c33} enforces that each UGV can recharge at most $d$ UAVs at a time (by making sure that at most one UAV recharges at one of the $d$ copies of a UGV vertex). Given a solution $x_l$, let $M_l$ represent the corresponding solution on graph $G$. Let us define $c(M) = c(x) = \sum c_{ij}x_{ij}$, $a(M) = a(x) = \sum a_{ij}x_{ij} $ for ease of notation.

\begin{remark}
In some applications, it might be desirable to have a constraint on the total detour time while maximizing the probability of robots not running out of charge. Note that although we define the problem with sum of detour times as the objective and with a constraint on probability, the discussion and results in this paper hold with this constraint and the objective swapped as well. 
\end{remark}

\subsection{Hardness}
We now characterize the hardness of Problem~\ref{pbm:original} in the following result.
\begin{lemma}
Problem~\ref{pbm:original} is NP-hard.
\end{lemma}
\begin{proof}
We show the hardness by using a reduction from the problem \textsc{EvenOddPartition} which is shown to be NP-complete in~\cite{johnson1979computers}.

In an instance of \textsc{EvenOddPartition}, we are given a finite set $\{z_1,\ldots,z_{2\ell}\}$ of positive integers, and the decision problem is to find whether there exists a subset $I\subseteq \{1,\ldots,2\ell\}$ such that $\sum_{i\in I}z_i = \sum_{i\notin I}z_i$ and $|I\cap \{z_{2j-1},z_{2j}\}|=1$ for each $j\in \{1,\ldots,\ell\}$.  

Given an instance of \textsc{EvenOddPartition}, we construct an instance of Problem~\ref{pbm:original} as follows. Consider a bipartite graph $G=(V_1\cup V_2,E)$ where $V_1$ consists of one vertex for each $j\in\{1,\ldots,\ell\}$ and $V_2$ consists of one vertex each for $j\in\{1,\ldots,2\ell\}$. Connect vertex $j\in V_1$ to vertex $2j \in V_2$ and to vertex $2j-1\in V_2$. Let us denote these edges by $e_{2j}$ and $e_{2j-1}$ respectively. The cost and weight on edge $e_{2j}$, i.e., $c(e_{2j})$ and $a(e_{2j})$ are $z_{2j}$ and $z_{2j-1}$ respectively. The cost and weight on edge $e_{2j-1}$ is $z_{2j-1}$ and $z_{2j}$ respectively. Exactly one outgoing edge from each vertex in $V_1$ should be selected in the solution. Finally we set the value of maximum allowed weight (equivalent to $B$ in Constraint~\eqref{eq:knapsack})  to $Z=\frac{1}{2}\sum_{i=1}^{2\ell}z_i$.

If the instance of \textsc{EvenOddPartition} is a \textsc{Yes} instance, then there exists a solution in the bipartite graph $G$ such that each vertex in $V_1$ has exactly one outgoing edge in the solution and the sum of weights on that solution is $Z$. Moreover, by the construction of $G$, the sum of costs of that solution is also $G$.

If the instance of \textsc{EvenOddPartition} is a \textsc{No} instance, then either the weight or the cost on any solution that has exactly one outgoing edge from each $j\in V_1$, is more than $Z$. Since the solutions with a weight of more than $Z$ are infeasible, the minimum cost of any feasible solution is more than $Z$. 

Therefore, if the optimal cost in the constructed bipartite graph is more than $Z$, then the corresponding \textsc{EvenOddPartition} is a \textsc{No} instance, and it is a \textsc{Yes} instance otherwise.
\end{proof}

We present a bicriteria approximation algorithm for Problem~\ref{pbm:original} in this paper. The main result regarding the algorithm is summarized below and we provide a detailed analysis in the next section.

\begin{theorem}
\label{thm:main}
For a given $\epsilon\in(0,1]$, there is a $(1+\epsilon,2)$-bicriteria approximation algorithm for Problem~\ref{pbm:original}, i.e., we get a solution $x$ such that 
\begin{itemize}
    \item $c(x)\leq (1+\epsilon)\texttt{opt} $, and
    \item $\sum_i x_{ij}\leq 2$,  \quad $\forall ~j\in V_g$. 
\end{itemize}
\end{theorem}
Note that the approximation algorithm gives a solution where one of the UGV's may have to recharge at most $d+1$ UAVs at a time, and since $d\geq 1$, it results in violation of Constraint~\eqref{eq:c33} by a factor of two as mentioned in Theorem~\ref{thm:main}.

{ \section{Algorithms and Analysis}}

Although Problem~\ref{pbm:original} can be solved using ILP solvers, as the number of variables in the problem increases, the runtime of the solver will become intractable. Since we need to solve this problem repeatedly in a receding horizon approach, an efficient solver is desirable. In this section, we present a series of algorithms to solve Problem~\ref{pbm:original}, that build on top of each other. Algorithm~\ref{alg:alg0} is a heuristic algorithm that uses binary search on the Lagrangian multiplier to find a feasible solution. Algorithm~\ref{alg:alg1} uses Algorithm~\ref{alg:alg0} to provide two solutions that are adjacent in the solution polytope, one of them being feasible. A bicriteria approximation algorithm is presented in Algorithm~\ref{alg:alg2}, which uses Algortithm~\ref{alg:alg1} as a subroutine. The bicriteria approximation algorithm may violate one of the constraints of the problem (one of the UGVs may need to recharge up to $d+1$ UAVs at a time). In practice, we can use Algorithm~\ref{alg:alg2} to solve the problem and if the solution violates the constraint, we can use the solution returned by Algorithm~\ref{alg:alg1}.

\begin{algorithm}[t]
	\caption{\textsc{BinarySearch}}
	\label{alg:alg0}
	\begin{algorithmic}[1]
	\Statex Input: Graph $G=(V_a\cup V_g,E)$ with edge costs $c$ and weights $a$, bound $B$
  \vspace{0.2em}
  \hrule
  \vspace{0.2em}
    \State Set a threshold $\Delta\lambda_{\texttt{min}}$
    \State $\lambda_l\leftarrow 0$, $\lambda_l\leftarrow \infty$
    \State Start from a positive value of $\lambda$
    \While {$\lambda_u-\lambda_l\geq\Delta\lambda_{\texttt{min}}$}
    \State Solve Problem~\ref{pbm:lagrange} using $\lambda$ to get solution $x$
    \If {$a(x)\leq B$}
    \State $M_1\leftarrow x$
    \State $\lambda_u\leftarrow \lambda$
    \State $\lambda\leftarrow {(\lambda_u+\lambda_l)}/{2}$
    \Else 
    \State $M_2\leftarrow x$
    \State $\lambda_l\leftarrow \lambda$
    \State $\lambda\leftarrow \min\{2\lambda,\lambda_u\}$
    \EndIf 
    \EndWhile
    \State \Return $M_1, M_2, \lambda$
    
	\end{algorithmic}
\end{algorithm}

\begin{algorithm}[t]
	\caption{\textsc{LocalSearch}}
	\label{alg:alg1}
	\begin{algorithmic}[1]
	\Statex Input: Graph $G=(V_a\cup V_g,E)$ with edge costs $c$ and weights $a$, bound $B$
  \vspace{0.2em}
  \hrule
  \vspace{0.2em}
    \State Get $\lambda$, $M_1$ and $M_2$ such that $a(M_1)\leq B \leq a(M_2)$ using Algorithm~\ref{alg:alg0} or~\cite{berger2011budgeted}\label{algln:lambda}
    \State $M'\leftarrow M_1\oplus M_2$
    \For {connected components $Y$ in $M'$}
    \If {$M'$ has one connected component}
    \State     \Return $M_1$, $M_2$, $\lambda$
    \Else
    \If {$a(M_1\oplus Y)\leq B$}
    \State $M_1\leftarrow M_1\oplus Y$
    \Else
    \State $M_2\leftarrow M_1\oplus Y$
    \EndIf
    \EndIf
    \State Remove $Y$ from $M'$
    \EndFor
	\end{algorithmic}
\end{algorithm}

\begin{algorithm}[ht]
	\caption{\textsc{RendezvousScheduling}}
	\label{alg:alg2}
	\begin{algorithmic}[1]
	\Statex Input: Graph $G=(V_a\cup V_g,E)$ with edge costs $c$ and weights $a$, bound $B$, $\epsilon\in(0,1]$
  \vspace{0.2em}
  \hrule
  \vspace{0.2em}
  \State $M^*_H\leftarrow$ Guess $p=\lceil\frac{1}{\epsilon}\rceil$ edges of highest cost in $M^*$
  \State $E\leftarrow E\setminus\{e:c(e)\geq \min_{e'\in M^*_H}{c(e')}\}$ 
  \For {$(i,j)\in M_H^*$}
  \State $V_a\leftarrow V_a\setminus\{\mathcal{V}_r:i\in\mathcal{V}_r\}$
  \State $V_g\leftarrow V_g\setminus \{j\}$
  \EndFor
  \State $B\leftarrow B - a(M^*_H)$
  \State Find $M_1$, $M_2$ and $\lambda$ from Algorithm~\ref{alg:alg1}
    \State $Z\leftarrow M_1\oplus M_2$ \label{algln:symm_diff}
    \For {edge $z_i$ in $Z=\{z_0,z_1,\ldots,z_{k-1}\}$}
    \If {$z_i\in M_1$} $\alpha_i=w_\lambda(z_i)$
    \EndIf
    \If {$z_i\in M_2$} $\alpha_i= - w_\lambda(z_i)$
    \EndIf
    \EndFor
    \State Find $z_i$ such that $\sum_{j=i}^{i+h}\alpha_{i\Mod k} \leq 0$ for all $h$
    \State Find the longest sequence $Z''$ starting from $z_i$ such that $a(M_1\oplus Z')\leq B$
    \If {Last edge of $Z''$ is not in $M_2$}
    \State Remove the last edge of $Z''$ 
    \EndIf
    \State $M_L\leftarrow M_1\oplus Z''$ 
    \If {$M_L$ has two edges connected to $\mathcal{V}_r$ for some $r$}
    \State Remove one of those edges from $M_L$
    \EndIf
    \If {$M_L$ has two edges $(i_1,j)$ and $(i_2,j)$ for a $j\in V_g$}
    \If {$\exists$ free $j'$ such that $a(i_\ell,j')\leq a(i_\ell,j)$ }
    \State $M_L \leftarrow M_L\setminus \{(i_\ell,j)\}\cup \{(i_\ell,j')\}$
    \EndIf
    \EndIf
    \State \Return $M_L\cup M^*_H$ \label{algln:return}
	\end{algorithmic}
\end{algorithm}

We start the analysis of our algorithms by considering the following problem :
\begin{problem}
\label{pbm:lagrange}
\begin{align}
     \min &~\sum c_{ij}x_{ij} +\lambda (\sum_{i,j}  a_{ij} x_{ij}-B),
\end{align}
such that Constraints~\eqref{eq:c22}, \eqref{eq:c33} and~\eqref{eq:integral} hold.
\end{problem}

Let $w_\lambda(M)=w_\lambda(x) = c(x) +\lambda a(x)$. Also, let $\mathcal{S}$ denote the polytope formed by the set of constraints~\eqref{eq:c22} and~\eqref{eq:c33}. We first show that Probelm~\ref{pbm:lagrange} is solvable in polynomial time. 
\begin{lemma}
\label{lem:polynomial}
Problem~\ref{pbm:lagrange} is solvable in $O(|E|\log{|V|} (|E| + |V| \log(V))$ time on graph $G(V_a\cup V_b,E)$ where $V=V_a\cup V_b$.
\end{lemma}
\begin{proof}
We can solve Problem~\ref{pbm:lagrange} using a minimum cost network flow problem. Given the bipartite graph $G=(V_a \cup V_g,E)$, introduce $N_a$ nodes $v_1,\ldots,v_{N_a}$. Connect the vertices in $\mathcal{V}_i$ to $v_i$ and connect all vertices $v_1,\ldots,v_{N_a}$ to a source vertex $s$. Connect all the vertices in $V_g$ to a sink vertex $q$. The cost of all the edges $(i,j)\in E$ is set to $c_{ij}+\lambda a_{ij}$. The cost of all other edges is zero. The capacity of all edges is one and $N_a$ amount of flow is to be sent from $s$ to $q$. This instance has $O(|E|)$ edges and $O(|V_a|+|V_b))$ vertices. A solution to this minimum cost network flow problem will satisfy Constraints~\eqref{eq:c22} and~\eqref{eq:c33}. Since the total flow and capacities are integers, there exist integer solutions satisfying Constraint~\eqref{eq:integral}, that can be found using network flow algorithms such as Orlin's algorithm~\cite{orlin1988faster} which runs in $O(|E|\log{|V|} (|E| + |V| \log(V))$ time on a graph with $|V|$ vertices and $|E|$ edges.
\end{proof}

Since Problem~\ref{pbm:lagrange} is solvable in polynomial time, the following observation enables us to use binary search in Algorithm~\ref{alg:alg0} to find two solutions $M_1$ and $M_2$ such that $a(M_1)\leq B\leq a(M_2)$.\footnote{Note that if no solution $M_2$ exists such that $a(M_2)>B$, then solving Problem~\ref{pbm:lagrange} with $\lambda=0$ solves Problem~\ref{pbm:original}.} 
\begin{observation}
Let $x_1$ and $x_2$ be the optimal solutions to Problem~\ref{pbm:lagrange} with Lagrangian multipliers $\lambda_1$ and $\lambda_2$, where $\lambda_1>\lambda_2$. Then $a(x_1)\leq a(x_2)$, and $c(x_1)\geq c(x_2)$. 
\end{observation}
Although the run time of Algorithm~\ref{alg:alg0} depends on the stopping threshold value $\Delta\lambda_{\texttt{min}}$, an optimal Lagrangian multiplier $\lambda$ and two solutions $x_1$ and $x_2$ to Problem~\ref{pbm:lagrange}, with $w_{\lambda}(x_1)=w_{\lambda}(x_2)$ and satisfying $a(x_1)\leq B \leq a(x_2)$ and $c(x_1)\geq c(x_2)$ can be found in polynomial time~\cite{berger2011budgeted},~\cite[Theorem 24.3]{schrijver1998theory}. We can use the procedure from ~\cite{berger2011budgeted} and~\cite[Theorem 24.3]{schrijver1998theory} or Algorithm~\ref{alg:alg0} in Line~ ~\ref{algln:lambda} of Algorithm~\ref{alg:alg1} to find $M_1$, $M_2$ and $\lambda$. We use Algorithm~\ref{alg:alg0} in experiments to find $M_1$ and $M_2$ as it is simple to implement and works well in practice as seen in Section~\ref{sec:exp}. Since $w_\lambda(M_1)\leq w_\lambda(M^*)$ where $M^*$ is the optimal solution to Problem~\ref{pbm:original} with cost $\texttt{opt}$, we have,
\begin{align}
\begin{split}
\label{eq:opt}
    w_\lambda(M_1) - \lambda B&\leq  w_\lambda(M^*) - \lambda B \\
    &\leq w_\lambda(M^*) - \lambda a(M^*) \\
    &= c(M^*)=\texttt{opt}.
\end{split}
\end{align}

We will need the following definition for the analyses of Algorithms~\ref{alg:alg1} and~\ref{alg:alg2}.
\begin{definition}
Given the bipartite graph $G=(V_a\cup V_g,E)$ where $V_a = \cup_{i=1}^{N_a}\mathcal{V}_i$, consider the bipartite graph $G'$ with the vertices for each robot merged together, i.e.,  $G'=(V_c\cup V_g,E')$ where $V_c = \{\mathcal{V}_1,\ldots,\mathcal{V}_{N_a}\}$, and $E$' has an edge between $\mathcal{V}_r$ and $j\in V_g$ if $E$ has an edge between some $i\in \mathcal{V}_r$ and $j$. Then with abuse of notation, we define a subgraph $H$ of $G$ as a \emph{connected component} if the edges of $H$ form a connected path or cycle in $G'$.
\end{definition}
Now we show a property of adjacent extreme points in the solution polytope $\mathcal{S}$ similar to a property of graph matchings shown in~\cite{balinski1974assignment}.
\begin{lemma}
\label{lem:adjacent}
Two vertices $x_1$ and $x_2$ of $\mathcal{S}$ are adjacent if and only if the symmetric difference of their corresponding solutions $M_1$ and $M_2$ contain exactly one connected component. 
\end{lemma}
\begin{proof}
The symmetric difference of $M_1$ and $M_2$ is a union of connected components (paths and cycles) because each edge in the symmetric difference can have a degree at most two. Suppose that $M_1\oplus M_2$ contains exactly one connected component. Then we can construct an objective function (edge costs) such that $M_1$ and $M_2$ are the only two optimal extreme point solutions as follows. Let the number of edges in $M_1 \oplus M_2$ from $M_1$ and $M_2$ be $h_1$ and $h_2$ respectively. Assign a cost of one to all edges in $(M_1\cap M_1\oplus M_2)$ , a cost of $h_1/h_2$ to all edges in $(M_2\cap M_1\oplus M_2)$, a cost of zero to all edges in $M_1\cap M_2$, and a negative cost to all other edges in the graph. Then $M_1$ and $M_2$ are the only two extreme point solutions in $\mathcal{S}$ for this cost function.

Now suppose that the symmetric difference has at least two connected components. Let $Z$ be one of those connected components and let $M_3=M_1\oplus Z$ and $M_4 = M_2\oplus Z$. $M_3$ and $M_4$ are also solutions in $\mathcal{S}$ and
\begin{equation*}
    \frac{1}{2}(x_1+x_2) = \frac{1}{2}(x_3+x_4).
\end{equation*}
So if there is any objective function for which $M_1$ and $M_2$ are optimal solutions, then the mid point also has the same cost, and hence $M_3$ or $M_4$ also has the same cost. Therefore, $M_1$ and $M_2$ are not adjacent extreme points in $\mathcal{S}$.
\end{proof}


A procedure to find two adjacent extreme point solutions in the solution polytope for a maximum weight matching problem is given in~\cite{berger2011budgeted}. We use a similar procedure for our problem to find two adjacent solutions in $\mathcal{S}$ in Algorithm ~\ref{alg:alg1}.
We have the following result regarding Algorithm~\ref{alg:alg1}.

\begin{lemma}
\label{lem:neighbors}
Algorithm~\ref{alg:alg1} returns $\lambda$ and two solutions $M_1$, $M_2$, such that they correspond to adjacent extreme points in $\mathcal{S}$ and 
\begin{itemize}
    \item $w_\lambda(M_1) = w_\lambda(M_2)$,
    \item $a(M_1)\leq B \leq a(M_2)$,
    \item $c(M_1)\geq c(M_2)$.
    \item $c(M_1)\leq \texttt{opt} + \lambda B$
\end{itemize}
\end{lemma}

\begin{proof}
If $M'=M_1\oplus M_2$ consists of more than one connected component, we can pick one of those connected components, say $Y$, and let $N = M_1\oplus Y$. If $a(N)\leq B$, replace $M_1$ by $N$. Otherwise replace $M_2$  by $N$. We can repeat this step until $M1\oplus M2$ consists of only pone connected component. Note that at each step, $|M_1\cap M_2|$ increases by at least one, so this procedure stops in at most $N_a$ steps. Moreover, $M_1$ always remains feasible, and by the optimality of $M_1$ and $M_2$ with respect to $w_\lambda$, $w_\lambda(M_1)$ and $w_\lambda(M_2)$ do not change during this process. The last inequality follows from the definition of $w_\lambda(M)$ and~\eqref{eq:opt}.
\end{proof}

The following result shows that given $M_1$, $M_2$ and $\lambda$ from Algorithm~\ref{alg:alg1}, we can get a solution $M$ that may violate at most one constraint and the cost of $M$ is within $c_{\texttt{max}}$ of the optimal, where $c_{\texttt{max}}$ is the largest edge cost in $G$. Lines~\ref{algln:symm_diff} to~\ref{algln:return} of Algorithm~\ref{alg:alg2} find such a solution.
\begin{lemma}
\label{lem:d+1}
Let $M^*$ be the optimal solution to Problem~\ref{pbm:original} with cost $\texttt{opt}$. There is a polynomial time algorithm that finds a scheduling $M$ (corresponding to a solution $x$) such that
\begin{itemize}
    \item $c(M)\leq \texttt{opt} + c_{\texttt{max}}$, and
    \item one of the UGVs may have to recharge at most $d+1$ UAVs at a time.
\end{itemize}
\end{lemma}

\begin{proof}
From Lemma~\ref{lem:neighbors} and Lemma~\ref{lem:adjacent}  we have two solutions $M_1$ and $M_2$ such that $w_\lambda(M_1)=w_\lambda(M_2)$, and $M_1 \oplus M_2$ contains exactly one connected component $Z=(z_0,z_1,\ldots,z_{k-1})$. Consider the sequence
\begin{equation*}
    \alpha_0 = \delta_0 w_\lambda(z_0), \alpha_1 = \delta_1 w_\lambda(z_1), \ldots, \alpha_{k-1} = \delta_{k-1} w_\lambda(z_{k-1}),
\end{equation*}
where $\delta_i=1$ if $z_i\in M_1$ and $\delta_i=-1$ if $z_i\in M_2$. Since $w_\lambda(M_1)=w_\lambda(M_2)$, $\sum_i \alpha_i = 0$. Moreover there exists an edge $z_i, i\in \{0,1,\ldots,k-1\}$, such that for any cyclic subsequence $(z_i,z_{(i+1) \Mod k}, \ldots, z_{(i+h) \Mod k)}$,
     $\sum_{j=i}^{i+h}\alpha_{i\Mod k} \leq 0$\cite[Lemma 3]{berger2011budgeted}. Therefore,
     \begin{equation}
     \label{eq:gasoline}
         \sum_{e\in Z\cap M_2}w_\lambda(e)-\sum_{e\in Z\cap M_1}w_\lambda(e) = \sum_{j=i}^{i+h}\alpha_{i\Mod k} \leq 0.
     \end{equation}
     
Consider the smallest such subsequence $Z'$ such that $a(M_1\oplus Z')\geq B$. Assume that $Z'$ contains at least two edges, otherwise $c(M_1)\leq c(M_2) +  c_{\texttt{max}}$ and we are done. Note that removing one edge from the end of $Z'$ will mean that $Z'$ is the longest sequence such that $a(M_1\oplus Z')\leq B$. Since the connected component is alternating, remove at least one and at most two edges from the end of the sequence $Z'$ to get $Z''$ such that the last edge in $Z''$ belongs to $M_2$. Also note that by construction, the first edge in $Z''$ also belongs to $M_2$. Therefore the first and last edges from $Z''$ will appear in $M=M_1\oplus Z''$ and $M$ may result in more than one edge connected to some set $\mathcal{V}_i$ or some $j\in V_g$.  If one of the robots has more than one connected edges in $M$ (because the first and last edges in $Z''$ are from $M_2$), we remove one of those edges. Note that $M$ can have one more edge connected to a UGV location as compared to $M_1$, i.e., $\sum_{i} x_{ij} \leq 2$. Note that this can happen at one of the two ends of $Z''$ as it is alternating. So one of the UGVs may have to recharge at most $d+1$ UAVs at a time. Also note that $a(M_1\oplus Z'')\leq B$. Now we lower bound the value of $c(M)$. Since we remove at most one edge from $M_1\oplus Z''$ to get $M$, 

\begin{equation*}
c(M)\leq c(M_1\oplus Z'')\leq c(M_1\oplus Z') + c_{\texttt{max}},
\end{equation*}

where the second inequality is due to the fact that we remove at most two edges from $Z'$ to get $Z''$.

Now,

\begin{align*}
    c(M_1\oplus Z')&=w_{\lambda}(M_1\oplus Z')-\lambda a(M_1\oplus Z')\\
    &=w_{\lambda}(M_1\oplus Z')-\lambda B - \lambda (a(M_1\oplus Z')-B)\\
    & \leq w_{\lambda}(M_1)-\lambda B - \lambda (a(M_1\oplus Z')-B)\\
    &\leq \texttt{opt}  - \lambda (a(M_1\oplus Z')-B)\\
    &\leq \texttt{opt},
    \end{align*}
    
    where the first inequality is due to~\eqref{eq:gasoline} and the second last inequality is due to~\eqref{eq:opt}. Hence

\begin{equation*}
    c(M)\leq \texttt{opt} + c_{\texttt{max}}.
\end{equation*}
\end{proof}

\begin{corollary}
If each UAV can recharge at atleast $N_a$ different UGV recharging locations, there is a polynomial time algorithm that finds a feasible scheduling $M$ satisfying Constraints~\eqref{eq:c22} and~\eqref{eq:c33} such that $a(M)\leq B+2a_{\texttt{max}}$ and $c(M)\leq \texttt{opt}+3c_{\texttt{max}}$.
\end{corollary}

Now we prove the main result regarding Algorithm~\ref{alg:alg2}.
\begin{proof}[Proof of Theorem \ref{thm:main}]
Given $\epsilon\in(0,1]$, guess $p=\lceil1/\epsilon\rceil$ edges with highest cost value of $c_{ij}$ in the optimal solution $M^*$. Let these edges be $M^*_H$. Remove $M^*_H$ and all the edges with costs higher than the lowest cost in $M^*_H$ from the graph. Also for $(i,j)\in M^*_H$ where $i\in \mathcal{V}_r$, remove the vertices $\mathcal{V}_r$ from $V_a$ and $j$ from $V_g$.  Also decrease $B$ by $a(M^*_H)$. Then if the optimal solution to the resulting instance is $M^*_L$, $M^*_H\cup M_L^*$ is the optimal solution to the original problem. The maximum cost of an edge in the resulting instance will be at most $c(M_H)/p\leq \epsilon c(M_H)$. Then by Lemma~\ref{lem:d+1}, we get a solution $M_L$ for the new instance such that
\begin{align*}
        c(M^*_H)+c(M_L)&\leq c(M^*_H) + c(M_L^*)+ \epsilon c(M^*_H)  \\
        &\leq c(M^*) +\epsilon c(M^*)\leq (1+\epsilon)\texttt{opt}. 
\end{align*}
Since one of the UGVs may have to recharge at most $d+1$ UAVs at a time, and because $d\geq 1$, we get a $(1+\epsilon,2)$-bicriteria approximation. The algorithm requires $O(N_a^{{1}/{\epsilon}})$ guesses for $M^*_H$.
\end{proof}

\section{Experimental Results}

In this section, we first present a qualitative example of the persistent monitoring mission. Next, we study how system parameters (various risk tolerances) influence the recharging behaviors between the UAVs and the UGVs and the task performances of the UAVs. Then, we compare the performance of our scheduling strategy with a baseline (greedy strategy) using the mean time before the first failure and the travel distance overhead as metrics. Moreover, we  empirically evaluate the performance of the proposed heuristic algorithm. All experiments in Section~\ref{sec:exp_sim} are conducted using Python 3.8 on a PC with the i9-8950HK processor. The baseline solver is Gurobi 9.5.0. 
\label{sec:exp}

\subsection{Experimental Setup}
We consider a team consisting of two UAVs and two UGVs. The task routes $\mathcal{T}_a$ and $\mathcal{T}_g$ used in the problem can be either generated jointly by some task planners similar to those in \cite{manyam2019cooperative, yu2018algorithms} or can be generated separately by different task planners. In our case study, the task of two UGVs is to persistently monitor the road nodes. The setup here is similar to our previous work \cite{shi2022risk} on Intelligence, Surveillance, and Reconnaissance (ISR) where the focus is on improving high-level solutions.

 The UAV and UGV move at 9.8 m/s and 4.5 m/s respectively based on the field test data \cite{shi2022risk} in our ongoing project.  The recharging process (swapping battery) takes 100 s. The UAV and UGV need to persistently monitor the task nodes on the route. We apply our recharging strategy in a receding horizon fashion: every two minutes, the UAVs-UGVs team solves the RRRP problem to decide the UAVs' recharging schedule for the next $T=2500$ seconds. For each UAV, the current position will be the first node when we construct the bipartite graph. If some UAV is on a detour, we do not replan until the UAV has finished its detour. 

We consider two sources of stochasticity in the energy consumption model of UAVs: weight and wind velocity contribution to longitudinal steady airspeed.
The deterministic energy consumption model of the UAV is a polynomial fit constructed from analytical aircraft modeling data, given as,
\begin{equation}
    P(\bm{v_\infty}) = b_0 + b_1 \bm{v_\infty} + b_2 \bm{v_\infty}^2 + b_3 \bm{v_\infty}^3 + b_4 \bm{w} + b_5\bm{v_\infty} \bm{w}, 
\end{equation}
where $b_0$ to $b_5$ are coefficients, and their experimental values are listed in Table \ref{table:coefficient}.
\begin{table}[ht]
\centering
\caption{Coefficients for stochastic energy consumption model}
\begin{tabular}[t]{lcccccc}
\toprule
&$b_0$ &$b_1$ &$b_2$ &$b_3$ &$b_4$ &$b_5$\\
\midrule
{Value} &-88.77 &3.53 &-0.42 &0.043  &107.5  &-2.74\\
\bottomrule
\end{tabular}
\label{table:coefficient}
\end{table}

Weight is randomly selected following a normal distribution with a mean of 2.3 kg and a standard deviation of 0.05 kg, $\bm{w} \sim \mathcal{N}(\mu_{\bm{w}}, \sigma^2_{\bm{w}})$.
Vehicle airspeed, $v_\infty$, is the sum of the vehicle ground speed, $v$, and the component of the wind velocity that is parallel to the vehicle ground speed, ignoring sideslip angle and lateral wind components.
\begin{equation}
    v_\infty = \lvert \Bar{v_g} + \rm{cos}(-\psi)\bm{\xi}_{a, b} \rvert 
\end{equation}
The longitudinal wind speed contribution is derived from two random parameters; wind speed, and wind direction.
Wind speed is modeled using the Weibull probability distribution model of wind speed distribution, $\bm{\xi}_{a, b}$, with a characteristic velocity $a=1.5$ m/s and a shape parameter $b=3$. This is representative of a fairly mild steady wind near ground level.
Wind direction $\psi$ is the heading direction of the wind, and is uniformly randomly selected on a range of $[0,360)$ degrees. 

\subsection{Simulation Results} \label{sec:exp_sim}
\begin{figure*}
    \centering
    \subfloat[]{
    \includegraphics[width=0.27 \textwidth]{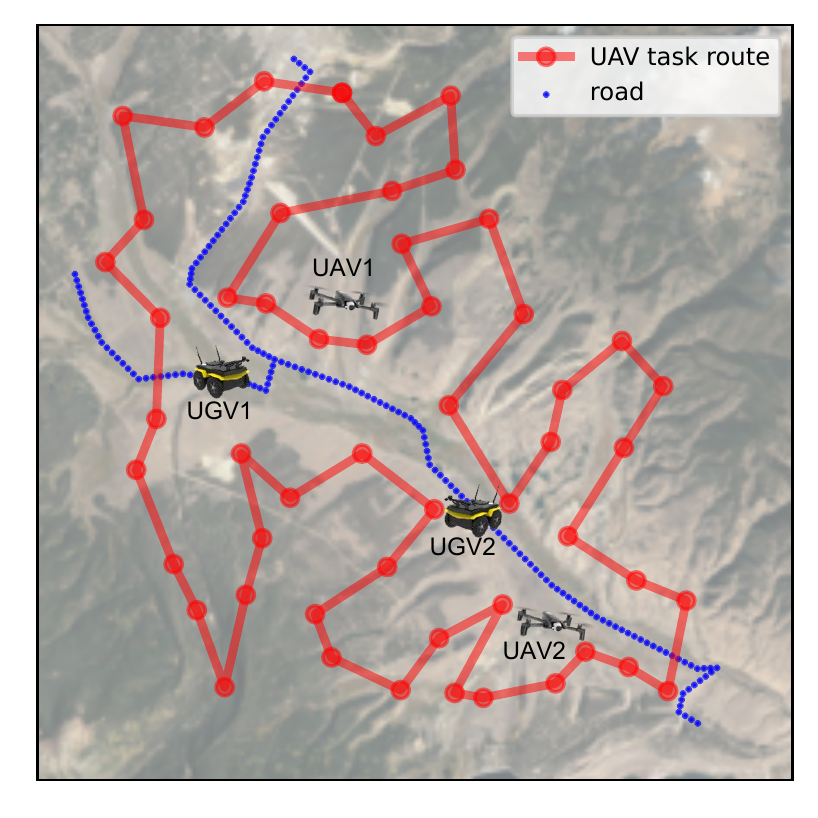}
    \label{fig:task_route}
    }
    \subfloat[]{
    \includegraphics[width=0.27 \textwidth]{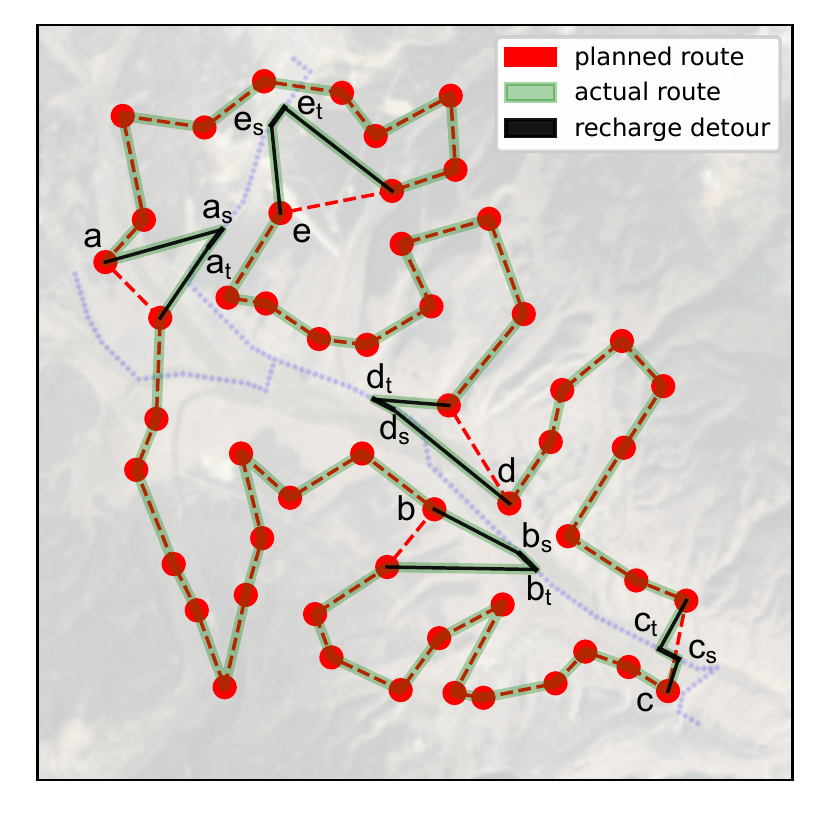}
    \label{fig:planned_actual_route_uav0_risk0.01}
    } 
    \subfloat[]{
    \centering
    \includegraphics[width=0.30 \textwidth]{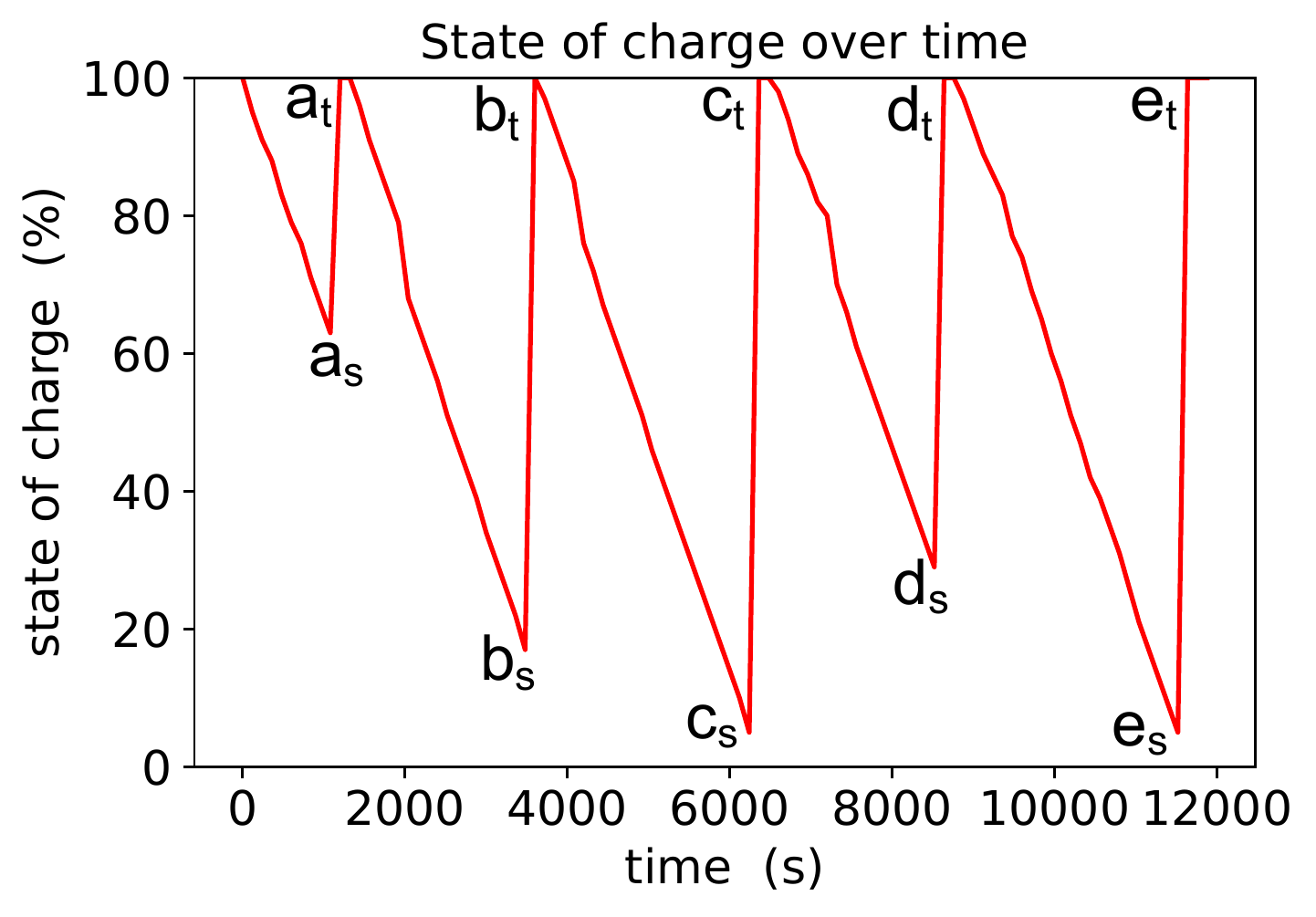}
    \label{fig:power_consumption}
    }
    \caption{
     A qualitative example to illustrate how UAV and UGV rendezvous with each other under the proposed scheduling strategy that is obtained by solving the RRRP\@. The risk tolerance is set to be $\rho=0.1$ in this case study. Subscriptions $s$ and $t$ denote the start and the terminal of the recharging process. (a) The input of the RRRP problem including the UAV and UGV tasks and the road network. (b) One sample tour of  UAV 1 when it persistently monitors the route. (c) One sample history of SOC for UAV 1. }
    \label{fig:rendezvous_illustrative_example}
\end{figure*}

\begin{figure}
    \centering
    \includegraphics[width=0.35 \textwidth]{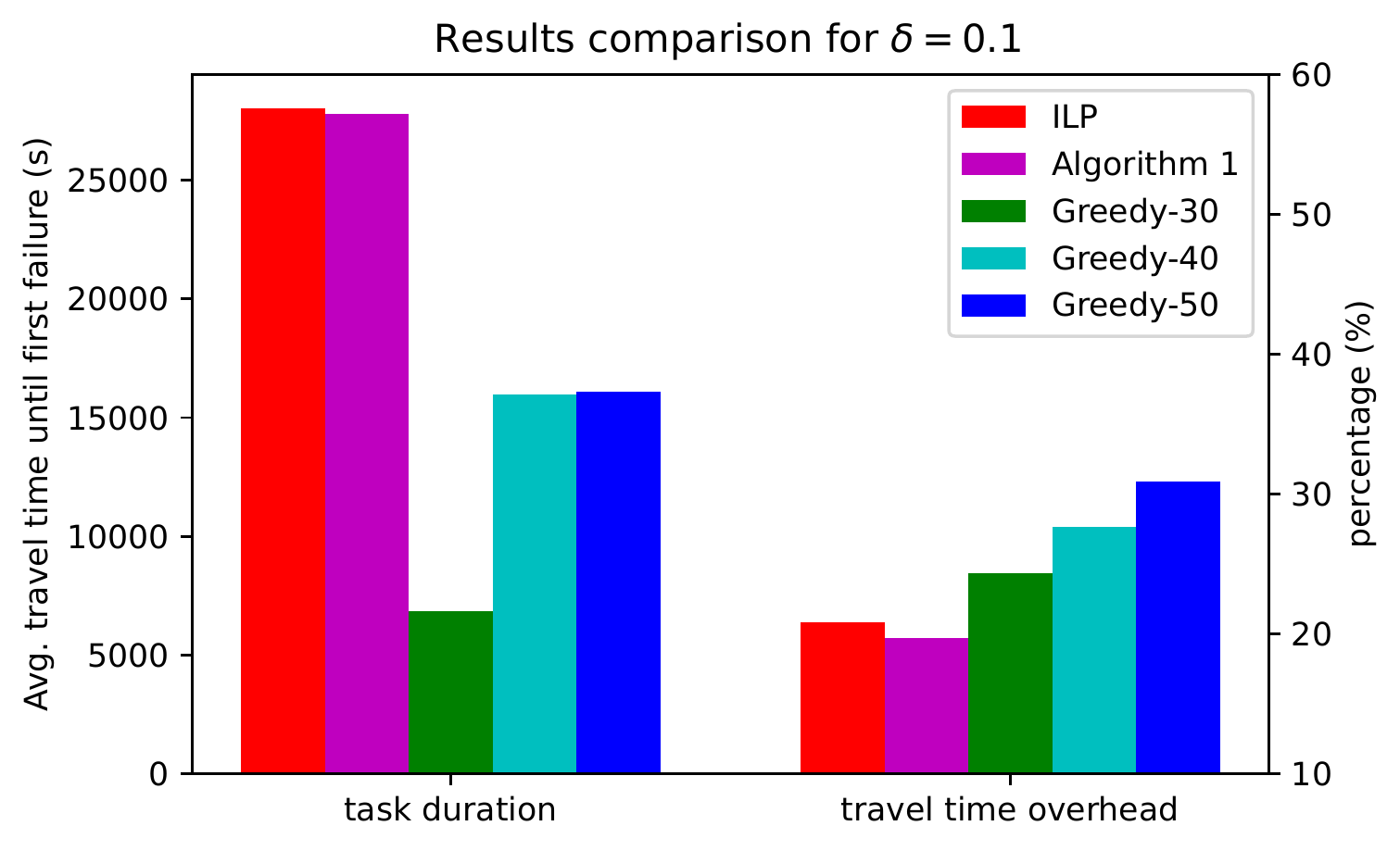}
    \label{fig:battery_history}
    \caption{
     Results comparisons for the RRRP scheduling strategy and the greedy strategies with $\rho=0.1$. ILP and Algorithm \ref{alg:alg0} refer to the solution returned by the ILP solver and Algorithm \ref{alg:alg0} by solving RRRP. }
    \label{fig:comparison_baseline}
\end{figure}
\begin{table}[t]
    \centering
    \caption{Statistical results for UAVs}
    \begin{tabular}[t]{lccc}
            \toprule
            UAV data &$\rho=0.01$ &$\rho=0.1$ &$\rho=0.3$ \\
            \midrule
             \shortstack[l]{Avg. travel time \\
            before out of charge (s)} &\centering 39660 &27600 &24360 \\
            \midrule
            Avg. travel time overhead &19.7 \% &18.5 \% &17.8 \% \\
            \midrule
            Avg. \# of task nodes visited &158 & 110  &105  \\
            \midrule
             \shortstack[l]{Avg. \#  of rendezvous \\
             per $T=2500$ s} &1.4 &1.3 &1.3 \\
            \bottomrule
    \end{tabular}
    \label{table:different_risk_threshold}
\end{table}

\paragraph{Qualitative Example} An illustrative example of the input and the output of the problem considered is shown in Figure \ref{fig:rendezvous_illustrative_example}. The input of the problem is shown in Figure \ref{fig:task_route}, which consists of UAV task nodes (red dots),  and nodes of the road network (blue nodes). 
Figure \ref{fig:planned_actual_route_uav0_risk0.01}  shows one tour route of one UAV when the system executes the proposed strategy in a receding horizon fashion. The UAV monitors the task route persistently. When the UAV reaches node $a$, it doesn't move forward to its next task node (connected through a dashed red line). Instead, the new schedule is to rendezvous with UGV at $a_s$ and takes off from the UGV at $a_t$, and then go to its next task node. Similarly, the UAV will rendezvous with the UGV when it is close to nodes $b,c,d$ and $e$. Subscriptions $s$ and $t$ denote the start and the terminal of the recharging process. A sample of the history of the state of charge (SOC) is shown in Figure \ref{fig:power_consumption}. We can observe in Figure \ref{fig:power_consumption} that the UAV's recharging strategy is more than a simple rule (for example get recharged when the SOC is below 50 $\%$) and may get recharged at various values of SOC.

\paragraph{Effect of Risk Tolerance} Next, we show how different risk tolerances influence the recharging behaviors under our RRRP formulation. In these experiments, we set the risk tolerance $\rho$ to be 0.01, 0.1, and 0.3.  Some statistical data are summarized in table \ref{table:different_risk_threshold}. We use four metrics to quantify the performance of the strategy:
\begin{itemize}
\item{Mean time before the first failure} here failure refers to the case when one UAV in the team is out of charge and needs human intervention. This quantity reflects how frequent the system needs human involved and we expect this quantity to be large enough. 

\item{Travel time overhead} this quantity is computed as 
    \begin{equation*}
        \frac{\text{actual~travel~time}-\text{task~time}}{\text{task~time}},
    \end{equation*}
    where task time is the time of the route obtained when we project the actual flight route into the planned route. This quantity reflects how well the UAV is performing its task and we expect this quantity to reasonably large. 
    
\item{Avg. \# of task nodes visited} the average number of task nodes visited by the UAV before its first failure. 

\item{Average number of rendezvous per planning horizon $T$} if this number is too large, it suggests that the UAV is scheduled too many recharging detours, which should be avoided.
\end{itemize}

In general, we observe that when the risk tolerance is set to be smaller, the mean time before the first failure will be longer. Similarly, the travel time overhead and the average number of rendezvous per planning horizon will be greater, which implies the UAV spends more portion of flight time in the recharging detours.

\paragraph{Comparison of Algorithm~\ref{alg:alg0} with Baseline} To validate that the scheduling strategy  constructed  from  RRRP, we compare our strategy with a greedy baseline. The greedy policy is set as: choose to rendezvous when state-of-charge drops below a set value.  We consider three set values $30\%$,  $40\%$, and $50\%$, and the corresponding strategies are denoted as \textit{Greedy-30}, \textit{Greedy-40}, and \textit{Greedy-50}. In this experiment, we set $\rho=0.1$. The first observation is that Algorithm \ref{alg:alg0} achieves close performances in both metrics compared to that of ILP. Second, as shown in Figure \ref{fig:comparison_baseline}, our strategy (obtained by both ILP and Algorithm \ref{alg:alg0}) can achieve longer travel time before first failure on average (left group) and a relatively lower travel distance overhead (right group), which implies that our strategy will not lead too many unnecessary rendezvous. Moreover Algorithm~\ref{alg:alg0} has a better travel time overhead than that of ILP although ILP solves RRRP optimally for a given horizon. This is likely due to solving RRRP repeatedly in a receding horizon manner. 

\paragraph{Scalability of Algorithm~\ref{alg:alg1}} We also compare the performance of Algorithm~\ref{alg:alg1} with an ILP solver empirically. We used Algorithm~\ref{alg:alg1} for comparison instead of Algorithm~\ref{alg:alg2} as Algorithm~\ref{alg:alg1} and the ILP always return a feasible solution, making the objective value comparison fair. The ILP solver used for this set of experiments is \texttt{intlinprog} function from MATLAB, and Algorithm~\ref{alg:alg1} was also implemented in MATLAB for fair comparison. Since ILP solves Problem~\ref{pbm:original} optimally, the cost returned by Algorithm~\ref{alg:alg1} is at least that of ILP solver. The percentage difference in the objective function values for different problem sizes is shown in Figure~\ref{fig:obj_comparison}. For each problem size, represented by the number of edges or variables, twenty random problem instances were created and the boxplot of resulting objective value difference is shown in the plot. On average, among all the instances, Algorithm~\ref{alg:alg1} was within $15\%$ of the optimal solution. Note that the performance of Algorithm~\ref{alg:alg1} improves as the number of variables increases. 

Figure~\ref{fig:runtime_comparison} shows the average runtime comparison between Algorithm~\ref{alg:alg1} and the ILP solver. Note that the y-scale is logarithmic. For smaller problem instances, both the algorithms solved the problem within a second, with ILP being faster, however, as the number of variables increases, ILP becomes much slower, with the runtime for ILP being up to seven times more than that of Algorithm~\ref{alg:alg1} for $60500$ variables. Note that there may be other solvers for ILP that have better run time, but since ILP is NP-complete, the exponential gap between run times is likely to continue as the number of variables increases.

\begin{figure}
    \centering
    \includegraphics[width=0.9 \linewidth]{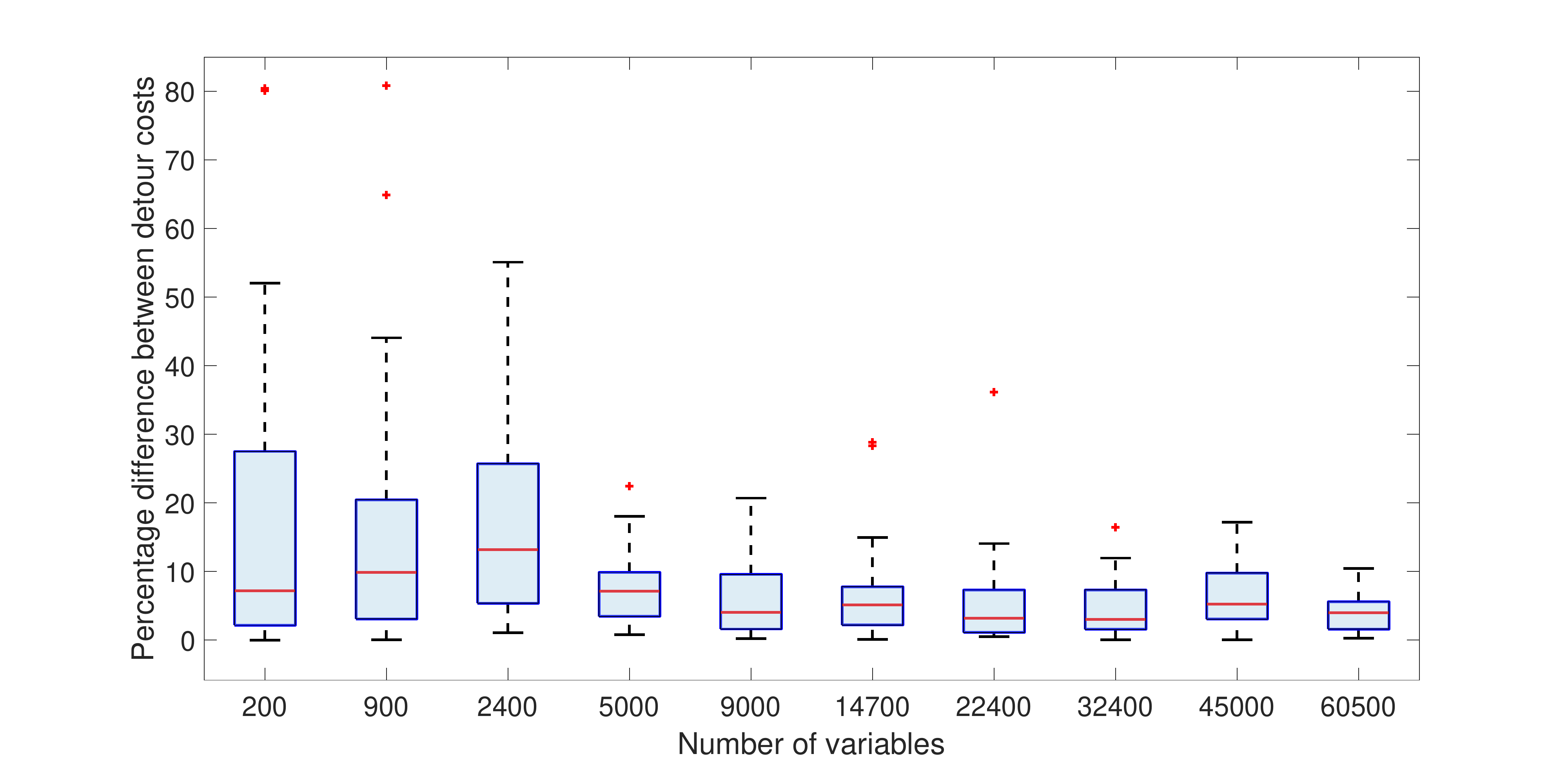}
    \caption{Percentage increase in the objective function of Algorithm~\ref{alg:alg1} as compared to an ILP solver. The boxplot shows the result of $20$ experiments for each problem size.}
     \label{fig:obj_comparison}
\end{figure}

\begin{figure}
    \centering
    \includegraphics[width=0.9 \linewidth]{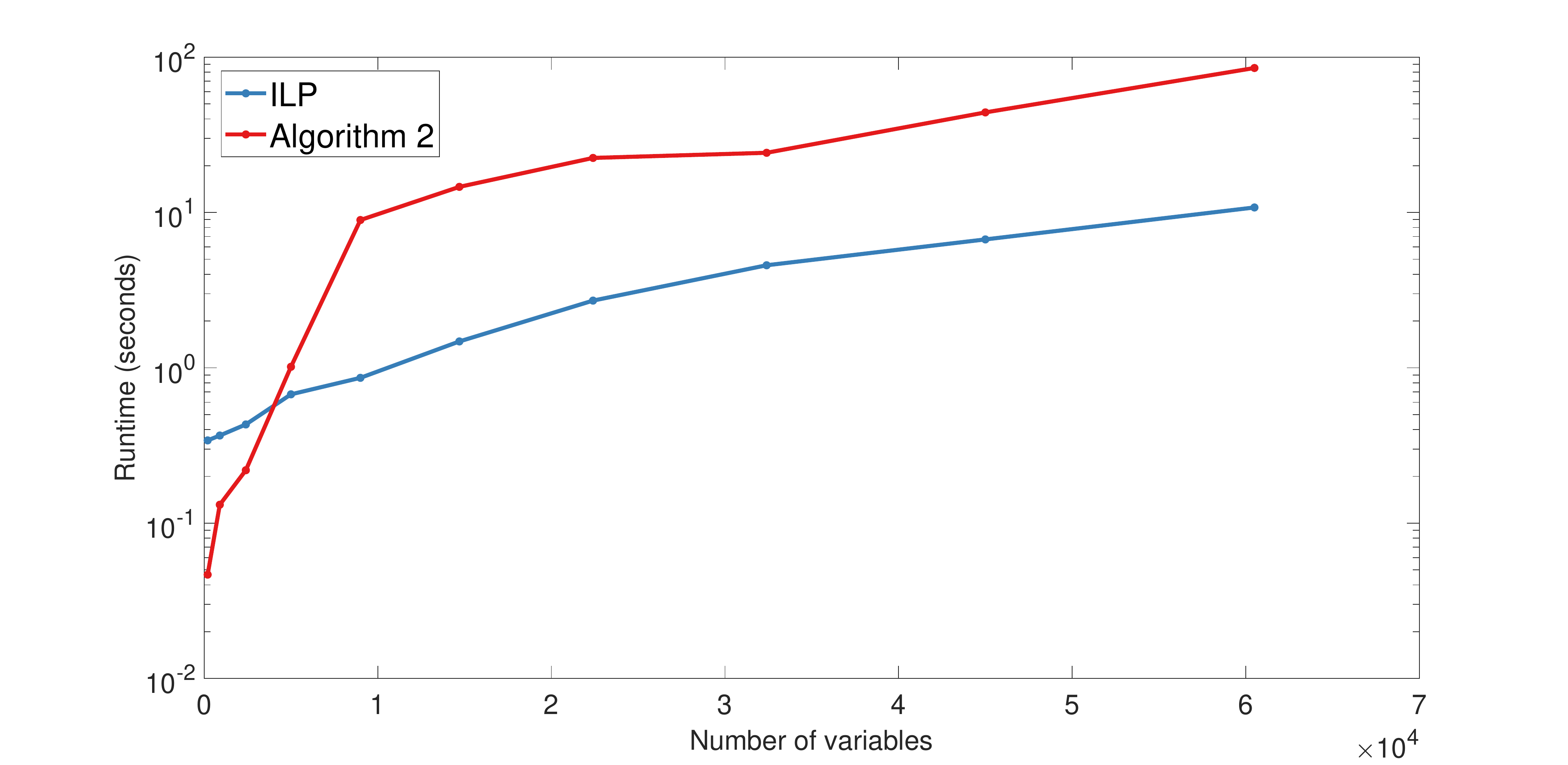}
    \caption{Comparison of runtimes of Algorithm~\ref{alg:alg1} and ILP solver. Note that the y-axis is logarithmic.} 
    \label{fig:runtime_comparison}
\end{figure}

\section{Conclusion}
In this paper, we study a resource allocation problem with matching and knapsack constraints for the UAVs-UGVs recharging rendezvous problem. We formulate this problem (Risk-aware Recharging
Rendezvous Problem (RRRP)) as one Integer Linear Program
(ILP) and propose  one bicriteria
approximation algorithm to solve RRRP. We validate our formulation and the proposed algorithm in one persistent monitoring application. In our current formulation, we assume that the task routes of vehicles are given. In future work, one direction that we will explore is to propose efficient and scalable routing algorithms to generate task routes for vehicles.

\bibliographystyle{IEEEtran}
\bibliography{ICRA2023}

\begin{appendix}
\begin{lemma}
\label{lem:UAV_vertices}
Given metric travel costs for the UAVs, there exists an optimal solution to Problem~\ref{pbm:original} where the UAVs will only leave their tours for recharging from either their current position or a task node.  
\end{lemma}
\begin{proof}
Suppose there exist an optimal solution with cost $c*$ where UAV $i$ has to leave its tour $\mathcal{T}^a_i$ from a point $p$, that is not its current location or a task node, to rendezvous with a UGV at location $g$. Let $v_1$ be the task node that precedes $p$ in $\mathcal{T}^a_i$, or the current position of UAV $i$, if $p$ is between the current position and first task node). Also let $v_2$ be the task node following $p$ in $\mathcal{T}^a_i$. Then using the triangle inequality, the total travel time for the edges $(v_1,g)$ and $(g,v_2)$ is not more than the the total travel time on edges  $(v_1,p), (p,g)$ and $(g,v_2)$. Moreover, since battery discharge depends on the time traveled by the UAV, the probability of the UAV not running out of charge for the detour taken from $v_1$ is not less than the probability of the UAV running out of charge for the detour taken from $p$. Hence, we get a solution where the UAV leaves its tour from its current position or a task node, with a cost at most $c^*$.
\end{proof}
\end{appendix}
\end{document}